\documentclass[final]{ecai}  
\usepackage{latexsym}

\ecaisubmission      

\usepackage[utf8]{inputenc}
\usepackage{graphicx}
\usepackage{latexsym}
\usepackage[T1]{fontenc}
\usepackage{float}
\usepackage{textcomp}
\usepackage{dsfont}
\usepackage{amsmath}
\usepackage{amsthm}
\usepackage{amssymb}
\usepackage{microtype}
\usepackage{subfigure}
\usepackage{booktabs} 
\usepackage{hyperref}
\usepackage{algorithm}
\usepackage{algorithmic}
\usepackage{balance}
\usepackage{hyperref}

\newtheorem{thrm}{Theorem}[section]
\newtheorem{lemma}[theorem]{Lemma}
\newtheorem{corollary}[theorem]{Corollary}
\newenvironment{proofsketch}{%
  \proof}{\endproof}
\theoremstyle{definition}
\newtheorem{definition}[theorem]{Definition}
\newtheorem{assumption}[theorem]{Assumption}

\newcommand{\indep}{\perp \!\!\! \perp}
\newcommand{\dep}{\not\perp \!\!\! \perp}

\global\long\def\pa{\mathrm{pa}}%

\global\long\def\nd{\mathrm{nd}}%

\global\long\def\argmax{\arg\!\max}%

\global\long\def\ci{\text{CI}}%

\begin{document}

\begin{frontmatter}

\title{Heteroscedastic Causal Structure Learning}
\author{\fnms{Bao}~\snm{Duong}\thanks{Corresponding Author. Email: \textit{duongng@deakin.edu.au}}}
\author{\fnms{Thin}~\snm{Nguyen}}

\address{Applied Artificial Intelligence Institute (A\textsuperscript{2}I\textsuperscript{2}), Deakin University, Australia}

\begin{abstract}

Heretofore, learning the directed acyclic graphs (DAGs) that encode
the cause-effect relationships embedded in observational data is a computationally challenging problem.
A recent trend of studies has shown that it is possible to recover
the DAGs with polynomial time complexity under the equal variances
assumption. However, this prohibits the heteroscedasticity of the
noise, which allows for more flexible modeling capabilities, but at
the same time is substantially more challenging to handle. In this
study, we tackle the heteroscedastic causal structure learning problem
under Gaussian noises. By exploiting the normality of the causal mechanisms,
we can recover a valid causal ordering, which can uniquely identifies
the causal DAG using a series of conditional independence tests. The
result is \textbf{HOST} (\textbf{H}eter\textbf{o}scedastic causal 
\textbf{ST}ructure learning), a simple yet effective causal structure
learning algorithm that scales polynomially in both sample size and
dimensionality. In addition, via extensive empirical evaluations on
a wide range of both controlled and real datasets, we show that the
proposed \textbf{HOST} method is competitive with state-of-the-art
approaches in both the causal order learning and structure learning
problems.
\end{abstract}
\end{frontmatter}

\section{Introduction\label{sec:Introduction}}

Causal structure learning provides the crucial knowledge to answer
not only the statistically descriptive questions of the systems of
interest, but also interventional and counterfactual queries. Thus,
it is at the heart of interest in many sciences where cause-effect
relationships are of utmost concerns, such as bioinformatics \cite{Sachs_etall_05Causal},
econometric \cite{Hunermund_Bareinboim_19Causal}, and neural sciences
\cite{Cao_etal_19Causal}. However, this problem poses serious challenges
when no prior knowledge is available, not to mention the limitation
of its practical solvability with respect to the scale of data.

Constraint-based \cite{Spirtes_Glymour_91Algorithm,Spirtes_etal_00Causation,Colombo_etal_12Learning}
and score-based methods \cite{Chickering_02Optimal,Ramsey_15Scaling,Ramsey_17Million,Zheng_etal_18DAGs,Zheng_etal_20Learning}
are conventionally two main approaches in causal structure learning.
The former branch leverages a series of statistical tests to eliminate
implausible connections between the variables of interest, but the
number of tests can grow exponentially with the number of variables.
Meanwhile, the latter group assigns dedicated scores to each candidate
causal graph based on data fitness, then performs greedy search or
continuous optimization for the optimal answer with respect to the
defined score, which is a NP-Hard problem since the space of possible
DAGs is super-exponentially large in the number of variables \cite{Robinson_77Counting}.
Both of these approaches cannot eliminate all improbable solutions
without further experimentations or prior knowledge, so the result
is a class of causal graphs that can induce the same observed data,
also known as the Markov Equivalent Class (MEC).

Recently, there is a novel line of works that concerns with additive
noise models where the noises have equal variances \cite{Buhlmann_etall_14Cam,Ghoshal_18Learning,Chen_etal_19Causal,Gao_etal_20Polynomial,Rolland_etal_22Score,Sanchez_etal_22Diffusion}.
This assumption allows for polynomial time algorithms that produces
unique graphs with provable guarantees in accuracy. Specifically,
they put more focus on finding the causal orderings of the causal
graph instead of the graph itself, since a correct causal ordering
can be used to efficiently trace back to the true causal graph \cite{Verma_Pearl_90Causal,Squires_Uhler_22Causal}. 

However, the equal variances assumption limits the modelling capabilities
in practice where the noise variances can fluctuate. Specifically,
here we exclusively focus on the heteroscedastic causal structure
learning problem, wherein the ``heteroscedascity'' reflects the fact
that the conditional variance of each variable given its causes is
non-constant and depends on the causes, in contrary with existing
additive noise models where it is constant across all variables in
the equal variances assumption aforementioned, or is constant for
each variable in the unequal variances models discussed in \cite{Buhlmann_etall_14Cam,Gao_etal_20Polynomial}.

While there have been progresses on this model \cite{Khemakhem_etal_21Causal,Xu_etal_22Inferring,Strobl_Lasko_22Identifying,Immer_etal_22Identifiability},
theoretical identifiability and methods are only specifically studied
for the case of two variables. As a result, they often propose to
extend to more than two variables by first recovering the skeleton
of the graph using conventional methods, then orient the undirected
edges with developed methods. This approach severely relies on the
quality of the skeleton recovery algorithm, which may scale non-polynomially
with dimensionality, while does not ensure acyclicity. A few ordering-based
approaches can address the heteroscedasticity, but under very restrictive
conditions, e.g., constant expected noise variances \cite{Gao_etal_20Polynomial}
or positive-valued noises in the multiplicative noise models \cite{Rajendran_etal_21Structure}.

\paragraph{Present study.}

For those reasons, in this work we propose to handle the generic heteroscedastic
causal structure learning problem from the causal ordering standpoint.
For simplicity, we first consider Gaussian noises. We introduce \textbf{HOST}\footnote{Source code is available at \url{https://github.com/baosws/HOST}.}
(\textbf{H}eter\textbf{o}scedastic causal \textbf{ST}ructure
learning), a method that operates in polynomial time and produces
unique graphs that are guaranteed to be acyclic. More specifically,
we exploit the conditional normality of each variable given its ancestral
sets and order the variables based on their conditional normality
statistics. Once a causal order is retrieved, an array of conditional
independence tests is deployed to restore the causal graph. To the
best of our knowledge, \textbf{HOST} is the first polynomial-time
method that handles generic heteroscedasticity.

\paragraph{Contributions.}

Our key contributions in this study are summarized as follows:
\begin{itemize}
\item We propose to tackle the heteroscedastic causal structure learning
problem under Gaussian noises by causal ordering. This is the first
attempt to handle the heteroscedasticity in causal models with polynomial-time,
to the best of our knowledge.
\item We present \textbf{HOST}, a causal structure learning method for heteroscedastic
Gaussian noise models that scales polynomially in both sample size
and dimensionality and finds unique acyclic DAGs.
\item We demonstrate the effectiveness of the proposed \textbf{HOST} method
on a broad range of both synthetic and real data under multiple crucial
aspects. The numerical evaluations confirm the quality of our method
against state-of-the-arts on both the tasks of causal ordering and
structure learning.
\end{itemize}

\section{Background\label{sec:Preliminaries}}

Let $X=\left(X_{1},\ldots,X_{d}\right)$ be the $d$-dimensional random
vector of interest, and $x=\left(x_{1},\ldots,x_{d}\right)\in\mathbb{R}^{d}$
be an observation of $X$. We use subscript indices for dimensions
and superscript indices for samples. For instance, $x_{i}^{\left(k\right)}$
indicates the $k$-th sample of $x_{i}$. $P\left(\cdot\right)$ represents
distributions and $p\left(\cdot\right)$ represents probability density
functions.

The causal structure can be described by a directed acyclic graph
(DAG) $\mathcal{G}=\left(\mathcal{V},\mathcal{E}\right)$ where each
vertex $i\in\mathcal{V}=\left[1..d\right]$ represents a random variable
$X_{i}$, and each edge $\left(j\rightarrow i\right)\in\mathcal{\mathcal{E}}$
indicates that $X_{j}$ is a direct cause of $X_{i}$. In this graph,
the \textit{parents} of a variable is defined as the set of its direct
causes, i.e., $\pa_{i}^{\mathcal{G}}:=\left\{ j\in\mathcal{V}\mid\left(j\rightarrow i\right)\in\mathcal{E}\right\} $.
Similarly, we denote $\nd_{i}^{\mathcal{G}}$ as the non-descendants
set of $X_{i}$ (excluding itself). Whenever it is clear from context,
we drop the superscript $\mathcal{G}$ to reduce notational clutter.

\subsection{Heteroscedastic Causal Models}

We consider the Structural Causal Model as follows
\begin{definition}
(Heteroscedastic Causal Model (HCM)). In a HCM, each variable is generated
by

\begin{equation}
X_{i}:=\mu_{i}\left(X_{\pa_{i}}\right)+\sigma_{i}\left(X_{\pa_{i}}\right)E_{i}\label{eq:hcm}
\end{equation}

where $E\sim\mathcal{N}\left(\mathbf{0};\mathbf{I}_{d}\right)$ is
the exogenous Gaussian noise vector.
\end{definition}

This model is also referred to as Location-Scale Noise Model (LSNM)
\cite{Immer_etal_22Identifiability} or Heteroscedastic Noise Model
(HNM) \cite{Strobl_Lasko_22Identifying} (except that in HNM, $\sigma$
models the conditional mean absolute deviation instead of the conditional
standard deviation).

In addition, for ease of analysis, we suppose the following regularity
constraints on the functions governing the causal processes:
\begin{assumption}
\label{assu:regularity}(Regularity conditions). We assume the following
for all $i=1..d$:
\begin{itemize}
\item $\mu_{i}$ and $\sigma_{i}$ are deterministic and differentiable.
\item $\sigma_{i}>0$ .
\item $\frac{\partial\mu_{i}}{\partial X_{j}}\not\equiv0$ and $\frac{\partial\sigma_{i}}{\partial X_{j}}\not\equiv0$
for all $j\in\pa_{i}$.
\end{itemize}
\end{assumption}

Assumption~\ref{assu:regularity} outlines the conditions required to eliminate any trivial degenerate cases. To be more specific, the first condition favors deterministic functions over stochastic functions so the source of randomness is explicitly captured in the noise variables. Additionally, we restrict the functions to be differentiable for modelling simplicity. The second condition ensures that the conditional variance of a variable depends on its parent variables, which ensures the heteroscedasticity of the model. Lastly, the third condition stipulates that the functions are non-constant with respect to any parent $j\in\textrm{pa}_i$.

Following this model, it is clear that {\small{}$P\left(X_{i}\mid X_{\pa_{i}}\right)=\mathcal{N}\left(\mu_{i}\left(X_{\pa_{i}}\right),\sigma_{i}^{2}\left(X_{\pa_{i}}\right)\right)$},
where $\mu_{i}\left(x_{\pa_{i}}\right)$ and $\sigma_{i}^{2}\left(x_{\pa_{i}}\right)$
respectively models the conditional mean and variance of each variable
given its parents.

This model nicely follows the stable causal mechanism postulation
\cite{Janzing_Scholkopf_10Causal} in the sense that the causal mechanism,
i.e., the conditional distribution of an effect given its causes,
is simple and easy to be described, as it is always a canonical normal
distribution. However, this is not enough for identifiability since
the vice versa may not hold. Therefore, to ensure the HCM model is
identifiable, we assume that normality is only achieved for the sufficient
parents set. More formally:
\begin{assumption}
\label{assu:stable-hcm}Let $C\subseteq\nd_{i}$. Denote $X_{i}\mid X_{C}\sim\mathcal{N}$
when $\frac{X_{i}-\mathbb{E}\left[X_{i}\mid X_{C}\right]}{\sqrt{\mathbb{V}\left[X_{i}\mid X_{C}\right]}}\sim\mathcal{N}\left(0,1\right)$.
We assume that
\begin{equation}
X_{i}\mid X_{C}\sim\mathcal{N}\Leftrightarrow\pa_{i}\subseteq C
\end{equation}
\end{assumption}

Assumption~\ref{assu:stable-hcm} can be explained as follows: At first, each noise variable $E_i$ is a Gaussian variable, and when combined with the parents through the structural assignment $X_i:=\mu_i(X_{\textrm{pa}_i})+\sigma_i(X_{\textrm{pa}_i})E_i$, the normality is disrupted. Essentially, this assumption implies that we can only recover normality by controlling all parent variables, that is, by conditioning on them. The reason for this is that, in certain cases, the influences from the parents may not affect the normality of $E_i$, resulting in $X_i$ being a Gaussian variable and potentially being incorrectly detected as a node without any parent. In brief, the aim of this assumption is to exclude such cases.

We now argue that this is not a strict condition and can be achieved
in general.

Let us define
\begin{align}
U_{i} & =\frac{X_{i}-\mathbb{E}\left[X_{i}\mid X_{C}\right]}{\sqrt{\mathbb{V}\left[X_{i}\mid X_{C}\right]}}\\
 & =\underbrace{\frac{\mu_{i}\left(X_{\pa_{i}}\right)-\mathbb{E}\left[X_{i}\mid X_{C}\right]}{\sqrt{\mathbb{V}\left[X_{i}\mid X_{C}\right]}}}_{a_{i}\left(X_{C}\right)}+\underbrace{\frac{\sigma_{i}\left(X_{\pa_{i}}\right)}{\sqrt{\mathbb{V}\left[X_{i}\mid X_{C}\right]}}}_{b_{i}\left(X_{C}\right)}E_{i}\label{eq:gaussian-transform}
\end{align}

Assumption~\ref{assu:stable-hcm} implicitly excludes all the functions
$\mu_{i}\left(\cdot\right)$ and $\sigma_{i}\left(\cdot\right)$ such
that $U_{i}\sim\mathcal{N}\left(0,1\right)$ when $\pa_{i}\not\subset C$.

Given $U_{i},E_{i}\sim\mathcal{N}\left(0,1\right)$, a trivial condition
for Eqn.~(\ref{eq:gaussian-transform}) to happen is when $b_{i}\left(X_{C}\right)$
is a constant $t$ and $a_{i}\left(X_{C}\right)\sim\mathcal{N}\left(0,1-t^{2}\right)$.
Here $t$ is constrained to $\left[0,1\right]$.

When $\pa_{i}\not\subset C$ there exists $j\in\pa_{i}\setminus C$.
Since $\mathbb{V}\left[X_{i}\mid X_{C}\right]$ is not a function
of $X_{j}$, differentiating $b_{i}\left(X_{C}\right)$ with respect
to $X_{j}$ gives
\begin{equation}
0=\frac{1}{\sqrt{\mathbb{V}\left[X_{i}\mid X_{C}\right]}}\frac{\partial\sigma_{i}}{\partial X_{j}}
\end{equation}

which contradicts Assumption~\ref{assu:regularity}. On the other
hand, if $a_{i}$ is constant then it must be zero since $\mathbb{E}\left[a_{i}\right]=\mathbb{E}\left[U_{i}\right]=0$.
This is precisely when $b_{i}\equiv1$. In other words, neither $a\left(X_{\pa_{i}}\right)$
nor $b\left(X_{\pa_{i}}\right)$ can be constant.

Now for all positive integers $n$, the $n$-th order moments of $U_{i}$
and $E_{i}$ must match since they belong to the same distribution:
\begin{align}
\mathbb{E}\left[E_{i}^{n}\right] & =\mathbb{E}\left[\left(a_{i}+b_{i}E_{i}\right)^{n}\right]\\
 & =\mathbb{E}\left[\sum_{k=0}^{n}\binom{n}{k}a_{i}^{n-k}b_{i}^{k}E_{i}^{k}\right]\\
 & =\sum_{k=0}^{n}\binom{n}{k}\mathbb{E}\left[a_{i}^{n-k}b_{i}^{k}\right]\mathbb{E}\left[E_{i}^{k}\right]\label{eq:constraint}
\end{align}

Since $\mathbb{E}\left[E_{i}^{k}\right]$ is a known constant for
all positive integers $k$, if exist, three functions $p\left(X_{C}\right),a_{i}\left(X_{C}\right),b_{i}\left(X_{C}\right)$
must be \textit{inextricably interwoven} because the term $\mathbb{E}_{p\left(X_{C}\right)}\left[a_{i}^{n-k}\left(X_{C}\right)b_{i}^{k}\left(X_{C}\right)\right]$
must satisfy Eqn.~(\ref{eq:constraint}) \textit{for} \textit{all
}positive integers $n$. When limited to non-constant differentiable
functions, the solution space of these functions is even more severely
limited. Therefore, in conclusion, we expect Assumption~\ref{assu:stable-hcm}
to hold in general.


Interestingly, this assumption implies a faithfulness consequence
between the DAG $\mathcal{G}$ and the joint distribution of the data,
which is characterized by the following Corollary:
\begin{corollary}
\label{cor:faithfulness}$X_{i}\dep X_{j}\mid X_{\pa_{i}\setminus j}$
for all $j\in\pa_{i}$.
\end{corollary}

\begin{proof}
Take any $j\in\pa_{i}$. If $X_{j}\indep X_{i}\mid X_{\pa_{i}\setminus j}$
then $P\left(X_{i}\mid X_{\pa_{i}\setminus j}\right)=P\left(X_{i}\mid X_{\pa_{i}\setminus j},X_{j}\right)=P\left(X_{i}\mid X_{\pa_{i}}\right)=\mathcal{N}\left(\cdot\right)$,
which contradicts Assumption~\ref{assu:stable-hcm} with $C=\pa_{i}\setminus j$.
\end{proof}

\subsection{Causal Ordering}

A valid causal order is an arrangement of the vertices in $\mathcal{G}$
so that a cause is always positioned before all of its effects. Given a DAG $\mathcal{G}$, a valid causal ordering of $\mathcal{G}$
is a permutation $\pi=\left[\pi_{1},\ldots,\pi_{d}\right]$ of the
sequence $\left[1,\ldots,d\right]$ such that $\pi_{j}\in\nd_{\pi_{i}}^{\mathcal{G}}$
for all $j<i$, or equivalently, $\pa_{i}\subseteq\pi_{<i}\subseteq\nd_{i}\;\forall i$.

\begin{algorithm}[t]
\caption{Algorithm for recovering the latents.\label{alg:extract-residuals}}

\begin{algorithmic}[1]

\REQUIRE{Dataset $\mathcal{D}=\left\{ \left(x_{C}^{\left(k\right)},x_{i}^{\left(k\right)}\right)\right\} _{i=1}^{n}$
and additional hyperparameters.}

\ENSURE{The latent variable $U_{i}$.}

\STATE Initialize two neural networks $\eta_{1}\left(X_{C};\theta\right)$
and $\ln\left(-\eta_{2}\left(X_{C};\theta\right)\right)$ with scalar
outputs.

\STATE Find the optimal parameter $\theta^{\ast}$ using, e.g., gradient
ascent, by maximizing{\scriptsize{}
\[
\theta^{\ast}=\argmax_{\theta\in\Theta}\mathbb{E}_{\mathcal{D}}\left[\eta_{1}\left(\cdot\right)X_{i}+\eta_{2}\left(\cdot\right)X_{i}^{2}+\frac{\eta_{1}{}^{2}\left(\cdot\right)}{4\eta_{2}\left(\cdot\right)}+\frac{1}{2}\ln\left(-2\eta_{2}\left(\cdot\right)\right)\right]
\]
}{\scriptsize\par}

\STATE \textbf{return} $U_{i}=\frac{X_{i}-\hat{\mu}_{i}\left(X_{C}\right)}{\hat{\sigma}_{i}\left(X_{C}\right)}$,
where
\begin{align*}
\hat{\mu_{i}}\left(X_{C}\right) & =-\frac{\eta_{1}\left(X_{C};\theta^{\ast}\right)}{2\eta_{2}\left(X_{C};\theta^{\ast}\right)}\text{, and}\\
\hat{\sigma_{i}}\left(X_{C}\right) & =\frac{1}{\sqrt{-2\eta_{2}\left(X_{C};\theta^{\ast}\right)}}
\end{align*}

\end{algorithmic}
\end{algorithm}

\begin{algorithm}[t]
\caption{Causal Ordering Algorithm.\label{alg:Causal-Ordering-Algorithm.}}

\begin{algorithmic}[1]

\REQUIRE{Dataset $\mathcal{D}=\left\{ x^{\left(k\right)}\right\} _{k=1}^{n}\in\mathbb{R}^{n\times d}$,
tolerance level $\epsilon\geq0$, and additional hyperparameters.}

\ENSURE{A causal order $\pi$.}

\STATE$\pi\leftarrow\left[\right]$.

\WHILE{$\left|\pi\right|<d$}

\FOR{$i\in\mathcal{V}\setminus\pi$}

\STATE Use Algorithm~\ref{alg:extract-residuals} to extract $U_{i}$
from $\left\{ \left(x_{\pi}^{\left(k\right)},x_{i}^{\left(k\right)}\right)\right\} _{k=1}^{n}$.\hfill$\triangleright$
Sec.~\ref{subsec:Extracting-the-latents}.

\STATE$w_{i}\leftarrow W\left(U_{i}\right)$, where $W$ is the Shapiro-Wilk
statistics.\hfill$\triangleright$ Sec.~\ref{subsec:Testing-normality}.

\ENDFOR

\STATE$w^{\ast}\leftarrow\argmax_{i\in\mathcal{V}\setminus\pi}w_{i}$.

\STATE$L\leftarrow\left\{ i\in\mathcal{V}\setminus\pi\mid w^{\ast}-w_{i}\leq\epsilon\right\} $.\hfill$\triangleright$
Sec.~\ref{subsec:Layer-Decomposition}.

\STATE Sort $L$ decreasing by $w$.

\STATE Concatenate $L$ after $\pi$.

\ENDWHILE

\STATE \textbf{return} $\pi$.

\end{algorithmic}
\end{algorithm}

\begin{algorithm}[t]
\caption{\textbf{HOST}: \textbf{H}eteroscedastic \textbf{O}rdering-based causal
\textbf{ST}ructure learning Algorithm.\label{alg:HOST}}

\begin{algorithmic}[1]

\REQUIRE{Dataset $\mathcal{D}=\left\{ x^{\left(k\right)}\right\} _{k=1}^{n}\in\mathbb{R}^{n\times d}$,
a conditional independence test $\text{CI}$, significance level $\alpha\in\left[0,1\right]$,
and additional hyperparameters.

$\ci\left(X,Y,Z\right)$ should return the $p$-value for the null
hypothesis $\mathcal{H}_{0}:X\indep Y\mid Z$.}

\ENSURE{The causal DAG $\mathcal{G}$.}

\STATE Use Algorithm~\ref{alg:Causal-Ordering-Algorithm.} to find
a causal order $\pi$.\hfill$\triangleright$ Sec.~\ref{subsec:Causal-Ordering}.

\STATE Initialize empty DAG $\mathcal{G}$ with vertices set $\mathcal{V}$.

\FOR{$j<i$}

\IF{$\ci\left(X_{\pi_{j}},X_{\pi_{i}},X_{\pi_{<i}\setminus\pi_{j}}\right)<\alpha$}

\STATE Add edge $\left(\pi_{j}\rightarrow\pi_{i}\right)$ to $\mathcal{G}$.\hfill$\triangleright$
Sec.~\ref{subsec:DAG-recovery}.

\ENDIF

\ENDFOR

\STATE \textbf{return} $\mathcal{G}$.

\end{algorithmic}
\end{algorithm}

The causal order is of great interest because learning the causal
structure from a causal order not only eases the necessary of the
acyclicity constraint of the learned directed graph, but it can also
be performed in polynomial runtime. Specifically, under our setting,
the causal structure can be uniquely identified from any valid causal
order by the following Lemma:
\begin{lemma}
\label{lem:Faithfulness} Given a causal order $\pi$ and joint distribution
$P_{X}$ consistent with the true DAG $\mathcal{G}=\left(\mathcal{V},\mathcal{E}\right)$ following Assumptions~\ref{assu:regularity} and \ref{assu:stable-hcm}.
Then, for all $j<i$,

\begin{equation}
\left(\pi_{j}\rightarrow\pi_{i}\right)\in\mathcal{E}\Leftrightarrow X_{\pi_{j}}\dep_{P_{X}}X_{\pi_{i}}\mid X_{\pi_{<i}\setminus\pi_{j}}
\end{equation}

where $<i$ represents the set of indices that come before $i$, specifically $1, 2, \ldots, i-1$.

\end{lemma}

\begin{proof}
The $\Rightarrow$ direction follows directly from Corollary~\ref{cor:faithfulness}.
We now prove the $\Leftarrow$ direction.

Suppose $\pi_{j}\in\pi_{<i}\setminus\pa_{\pi_{i}}$. Since $\pa_{\pi_{i}}\subseteq\pi_{<i}$,
by the Causal Markov condition we have {\scriptsize{}
\begin{align}
P\left(X_{\pi_{i}}\mid X_{\pa_{\pi_{i}}}\right) & =P\left(X_{\pi_{i}}\mid X_{\pa_{\pi_{i}}},X_{\pi_{<i}\setminus\pa_{\pi_{i}}\setminus\pi_{j}}\right)\\
 & =P\left(X_{\pi_{i}}\mid X_{\pi_{<i}\setminus\pi_{j}}\right)\text{, and}\\
P\left(X_{\pi_{i}}\mid X_{\pa_{\pi_{i}}}\right) & =P\left(X_{\pi_{i}}\mid X_{\pa_{\pi_{i}}},X_{\pi_{<i}\setminus\pa_{\pi_{i}}\setminus\pi_{j}},X_{\pi_{j}}\right)\\
 & =P\left(X_{\pi_{i}}\mid X_{\pi_{<i}\setminus\pi_{j}},X_{\pi_{j}}\right)
\end{align}
}{\scriptsize\par}

Therefore, $X_{\pi_{i}}\indep X_{\pi_{j}}\mid X_{\pi_{<i}\setminus\pi_{j}}$,
which completes the proof.
\end{proof}
Additionally, the DAG recovered this way is unique by construction.

This property suggests performing a series of conditional independence
(CI) tests to recover the DAG with only $\mathcal{O}\left(d^{2}\right)$
tests, compared with an exponential number of tests in the classical
PC algorithm \cite{Spirtes_Glymour_91Algorithm}.

\section{HOST: Heteroscedastic Ordering-based Causal~Structure~Learning\label{sec:Method}}

Lemma~\ref{cor:faithfulness} has established that a causal ordering
is sufficient to recover the underlying DAG. Therefore, what remains
is how to retrieve such orderings. In this section, we explain in
details the causal ordering procedure and related technical considerations.

\subsection{Causal Order Identification\label{subsec:Causal-Ordering}}

Identifying the precise causal order is usually a challenging computational task. Nonetheless, certain assumptions can make the problem more manageable. Typically, the primary source of computational advantage in polynomial algorithms is the ability to identify a source or sink node in polynomial time. This is usually accomplished by assuming the ``equal variance'' condition in additive noise models (\cite{Chen_etal_19Causal}, \cite{Gao_etal_20Polynomial}), which allows a source node to be effectively identified among the remaining variables since it has the smallest conditional variance. However, in our study, this condition is not applicable due to the heteroscedastic nature of the model under consideration. Instead, we rely on the normality of the residuals to detect a source node, based on the assumption that only source nodes will have Gaussian residuals (Assumption~\ref{assu:stable-hcm}). Therefore, at each step, we can efficiently extract the residuals and test their normality in polynomial time. By repeating this process until all variables have been examined, we obtain a polynomial time algorithm for causal ordering.

Indeed, here we show that a valid causal order is fully identifiable under
Assumption~\ref{assu:stable-hcm}. To begin with, the following Lemma says that by conditioning on an
ancestral set, one can identify the subsequent source nodes in the
reduced graph where the ancestral set is removed, wherein the case
of empty ancestral set allows detecting source nodes in the original
graph.
\begin{lemma}
\label{lem:conditional-normality}Under assumptions \ref{assu:regularity} and \ref{assu:stable-hcm}, let $C$ be an ancestral set in
$\mathcal{G}$ (i.e., $C$ satisfies: $C\subseteq\bigcap_{i\not\in C}\nd_{i}^{\mathcal{G}}$).
Define $\mathcal{G}\setminus C$ as a reduced DAG with vertices set
$\mathcal{V}\setminus C$ and edges set $\left\{ \left(j\rightarrow i\right)\in\mathcal{E}\mid j,i\in\mathcal{V}\setminus C\right\} $.
For $i\not\in C$, if $X_{i}\mid X_{C}\sim\mathcal{N}$ then $i$
is a source node in $\mathcal{G}\setminus C$.
\end{lemma}

\begin{proof}
Since $X_{i}\mid X_{C}\sim\mathcal{N}$, we have $\pa_{i}^{\mathcal{G}}\subseteq C$
by Assumption~\ref{assu:stable-hcm}. Thus, $\pa_{i}^{\mathcal{G}\setminus C}=\emptyset$,
i.e., $i$ is a source node in $\mathcal{G}\setminus C$.
\end{proof}
Therefore, by leveraging Lemma~\ref{lem:conditional-normality} we
can derive a causal ordering algorithm specialized for HCM models.
Algorithm~\ref{alg:Causal-Ordering-Algorithm.} demonstrates the
key steps of our causal ordering algorithm, in which we iteratively
employ normality testing subroutines to detect new source nodes every
step. The latents extraction and normality testing components are
discussed more thoroughly in the following subsection.

\subsection{Identifying Source Nodes via Normality~Statistics}

We have shown that testing for $X_{i}\mid X_{C}\sim\mathcal{N}$ is
essential to the causal ordering procedure. More precisely, given
an ancestral set $C$ and a variable $X_{i}$, we need to test if
$\frac{X_{i}-\mathbb{E}\left[X_{i}\mid X_{C}\right]}{\sqrt{\mathbb{V}\left[X_{i}\mid X_{C}\right]}}\sim\mathcal{N}\left(0,1\right)$.
Again, let us denote $\frac{X_{i}-\mathbb{E}\left[X_{i}\mid X_{C}\right]}{\sqrt{\mathbb{V}\left[X_{i}\mid X_{C}\right]}}$
by $U_{i}$, which we term as the ``latent variable'' as argued. Then,
the problem is reduced to the conventional normality testing problem
with the null hypothesis $\mathcal{H}_{0}:U_{i}\sim\mathcal{N}\left(0,1\right)$
and alternative hypothesis $\mathcal{H}_{1}:U_{i}\not\sim\mathcal{N}\left(0,1\right)$.
We next show how to extract the latents and test for their normality
in our framework.

\subsubsection{Extracting the Latents\label{subsec:Extracting-the-latents}}

To excerpt $U_{i}$ from $X_{i}$ and $X_{C}$, it is natural to estimate
$\mathbb{E}\left[X_{i}\mid X_{C}\right]$ and $\mathbb{V}\left[X_{i}\mid X_{C}\right]$
and plug the estimates into the expression of $U_{i}$. These estimations,
which are nonlinear regression problems, can be done separately, for
example, as in \cite{Strobl_Lasko_22Identifying}. However, that means
several regression stages must be performed sequentially, which possibly
become computationally involved.

Therefore, we prefer to jointly estimate the conditional expectation
and standard deviation instead. A naive and common approach for this
would be directly parametrizing the conditional mean and standard
deviation using neural networks, such as in \cite{Khemakhem_etal_21Causal},
e.g.,
\begin{align}
t_{i}\left(X_{C};\theta\right) & \approx\mathbb{E}\left[X_{i}\mid X_{C}\right]\\
s_{i}\left(X_{C};\theta\right) & \approx\sqrt{\mathbb{V}\left[X_{i}\mid X_{C}\right]}
\end{align}

under the model $X_{i}\sim\mathcal{N}\left(t_{i}\left(X_{C};\theta\right),s_{i}^{2}\left(X_{C};\theta\right)\right)$.
The optimal parameters $\theta^{\ast}$ can be found by maximizing
the Gaussian log likelihood:{\footnotesize{}
\begin{align}
\ln p_{\theta}\left(X_{i}\mid X_{C}\right) & =\ln\mathcal{N}\left(X_{i};t_{i}\left(X_{C};\theta\right),s_{i}^{2}\left(X_{C};\theta\right)\right)\\
 & =-\frac{\left(X_{i}-t_{i}\left(X_{C};\theta\right)\right)^{2}}{2s_{i}^{2}\left(X_{C};\theta\right)}-\ln s_{i}\left(X_{C};\theta\right)\label{eq:naive-objective}\\
\mathcal{L}\left(\theta\right) & =\mathbb{E}\left[\ln p_{\theta}\left(X_{i}\mid X_{C}\right)\right]\\
\theta^{\ast} & =\argmax_{\theta}\mathcal{L}\left(\theta\right)
\end{align}
}{\footnotesize\par}

However, Eqn.~(\ref{eq:naive-objective}) is not a jointly concave
objective function w.r.t. $t_{i}$ and $s_{i}$. This is because its
Hessian matrix is not always negative-definite, since
\begin{equation}
\frac{\partial^{2}\ln p_{\theta}\left(X_{i}\mid X_{C}\right)}{\partial s_{i}^{2}}=\frac{-3\left(X_{i}-t_{i}\left(X_{C};\theta\right)\right)^{2}}{s_{i}^{4}\left(X_{C};\theta\right)}+\frac{1}{s_{i}^{2}\left(X_{C};\theta\right)}
\end{equation}

can be positive when $s_{i}^{2}\left(X_{C};\theta\right)>3\left(X_{i}-t_{i}\left(X_{C};\theta\right)\right)^{2}$.
Hence, jointly optimizing for them both via--e.g., common gradient-based
solvers--\textit{does not guarantee the global maxima}, even if the
neural networks have infinite capacities.

To mitigate this issue, following \cite{Le_etal_05Heteroscedastic,Immer_etal_22Identifiability},
we adopt the natural parametrization of the Gaussian distribution.
Particularly, we parametrize $\mathcal{N}\left(\mu,\sigma^{2}\right)$
using two natural parameters $\eta_{1},\eta_{2}$ such that $\mu=-\frac{\eta_{1}}{2\eta_{2}}$
and $\sigma^{2}=-\frac{1}{2\eta_{2}}$. The log likelihood now becomes
\begin{align}
\ln p_{\theta}\left(X_{i}\mid X_{C}\right) & =\eta_{1}X_{i}+\eta_{2}X_{i}^{2}+\frac{\eta_{1}^{2}}{4\eta_{2}}+\frac{1}{2}\ln\left(-2\eta_{2}\right)\label{eq:natural-parametrization}
\end{align}

where $\eta_{1}\left(X_{C};\theta\right)$ and $\eta_{2}\left(X_{C};\theta\right)$
are functions of $X_{C}$ that are parametrized by $\theta$, and
$\eta_{2}\left(X_{C};\theta\right)<0$. One can then show that the
objective function~(\ref{eq:natural-parametrization}) is now jointly
concave in both $\eta_{1}$ and $\eta_{2}$ \cite{Le_etal_05Heteroscedastic},
which makes gradient-based solutions to the maximum likelihood objective
consistent.

To adapt to arbitrarily nonlinear relationships, $\eta_{1}$ and $\eta_{2}$
can be parametrized with neural networks. More specifically, we can
parametrize $\eta_{1}\left(X_{C};\theta\right):\mathbb{R}^{\left|C\right|\times\left|\Theta\right|}\rightarrow\mathbb{R}$
as a simple Multiple Layer Perceptron (MLP) with parameters space
$\Theta$. Regarding $\eta_{2}$, since it must be negative, we should
instead parametrize $\ln\left(-2\eta_{2}\left(X_{C};\theta\right)\right):\mathbb{R}^{\left|C\right|\times\left|\Theta\right|}\rightarrow\mathbb{R}$
as another MLP. However, for merely linear maps $\eta_{1}$ and $\ln\left(-2\eta_{2}\right)$,
the conditional distribution $P_{\theta}\left(X_{i}\mid X_{C}\right)$
can still have nonlinear expectations and standard deviations, which
already represent a wide class of distributions.

To summarize, the latents extraction subroutine is described in Algorithm~\ref{alg:extract-residuals}.

\subsubsection{Testing for the Latents' Normality\label{subsec:Testing-normality}}

Having the latents retrieved, we are now in a position to test for
their normality to detect source nodes.

Normality testing is a well-studied problem where a broad variety
of methods is available. For an overview, see, e.g., \cite{Thode_02Testing,Das_Imon_16Brief}.
In this study, we particularly employ the well-regarded Shapiro-Wilk
test \cite{Shapiro_Wilk_65An} since it has been shown to have a better
power against other common alternative approaches \cite{Razali_etal_11Power}.
That being said, the considering component of our method is modular
and any valid normality testing method can be employed in place of
the Shapiro-Wilk test.

We offer a brief explanation of the Shapiro-Wilk test in Appendix~\ref{subsec:Shapiro-Wilk-test}.
Simply put, the test's statistics of the Shapiro-Wilk test, denoted
as $W$, has the range of $\left[0,1\right]$ where the value of one
indicates perfect normality and the value of zero suggests strong
non-normality.

We can now detect new source nodes using the Shapiro-Wilk test. To
proceed, it is intuitive to adopt the whole hypothesis testing procedure
for each remaining variable. More specifically, with a significance
level $\alpha$ chosen prior to seeing the data, we compute the test
statistics $W$ for each variable and their associated $p$-values,
then select those whose $p$-values less than $\alpha$.

However, the null distribution of $W$ is complicated and unknown
\cite{Shapiro_Wilk_65An}, thus computing its $p$-value requires
Monte Carlo simulation, which can be computationally intense. Moreover,
in practice there may be no variable with $p\text{-value}<\alpha$,
making the iterative process unhalted. Hence, we take a slight detour
by selecting the variable with the highest $W$ statistics, which
does not require $p$-value calculation and always exists.

\subsubsection{Layer Decomposition\label{subsec:Layer-Decomposition}}

\begin{figure*}[t]
\begin{centering}
\begin{tabular}{c}
\includegraphics[width=\textwidth]{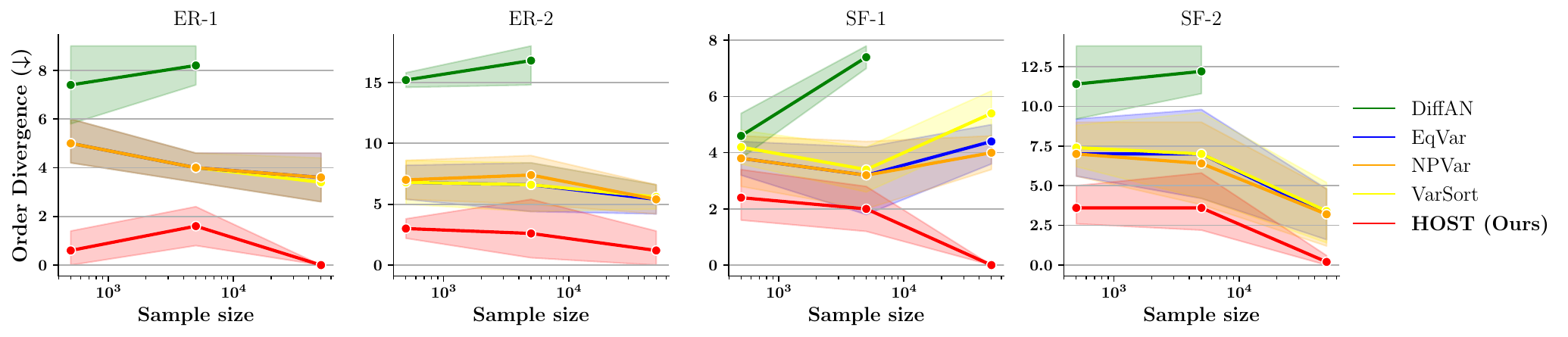}\tabularnewline
(a) Order Divergence (lower is better) with different Sample sizes.\tabularnewline
\includegraphics[width=1\textwidth]{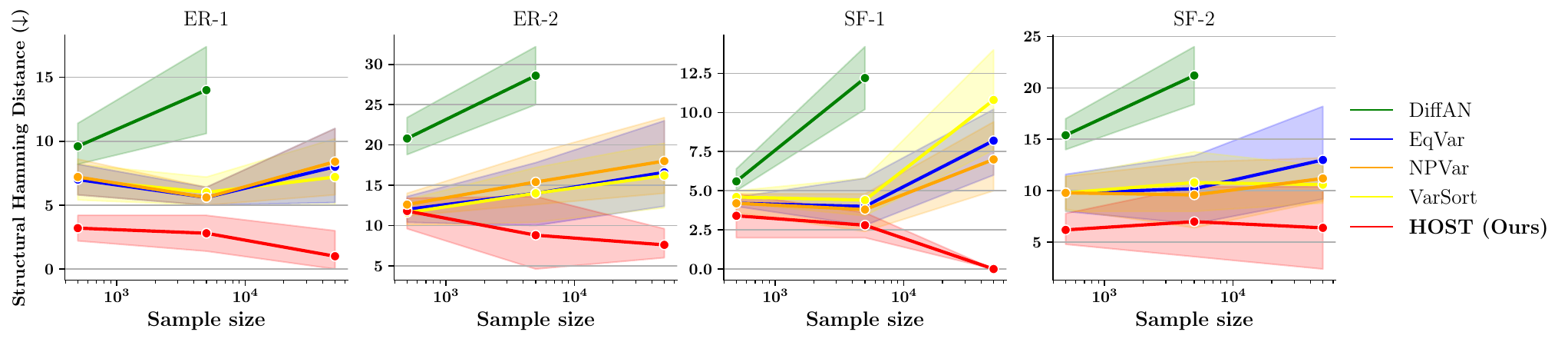}\tabularnewline
(b) Structural Hamming Distance (lower is better) with different Sample
sizes.\tabularnewline
\end{tabular}
\par\end{centering}
\centering{}\caption{Causal structure learning performance on synthetic data as function
of Sample size under Linear parameterization. We fix $d=10$ and vary
the sample size. The proposed \textbf{HOST} method is compared against
DiffAN \cite{Sanchez_etal_22Diffusion}, EqVar \cite{Chen_etal_19Causal},
NPVar \cite{Gao_etal_20Polynomial}, and VarSort \cite{Reisach_etall_21Beware}. Column: Graph
type. Shaded areas are 95\% confidence intervals over five independent
runs. LCIT~\cite{Duong_Nguyen_22Conditional} is used to recover
the DAGs from the causal orders. Missing data of DiffAN is due to
overly excessive runtime.\label{fig:sample-size-linear}}
\end{figure*}

For sparse graphs, each iteration may find multiple valid source nodes.
If only one is chosen then the others must have their calculations
re-done multiple times, which is wasteful of resource, and the chance
of miscalculation may arise in the subsequent steps as the ancestral
set increases in cardinality. Therefore, we should select as many
source nodes as possible in each step to avoid these pitfalls.

To do this, we employ a tolerance threshold $\epsilon$ onto the selection
of the source nodes based on their $W$ statistics. Particularly,
let $w^{\ast}$ be the largest $W$ statistics value of the remaining
variables in each step, we will include all variables with $W$ within
the $\epsilon$-radius of $w^{\ast}$ as new source nodes, which form
a ``layer'' similarly to \cite{Gao_etal_20Polynomial}. 

Further, there is a trade-off between accuracy and runtime when choosing
$\epsilon$. If $\epsilon$ is too small then the effect of computation
reduction is negligible, whereas larger $\epsilon$ will allow non-source
nodes into the layer. In our implementation, we choose a value of
$10^{-4}$, which is relatively small in comparison with the range
$\left[0,1\right]$ of $W$, since accuracy is preferred in our experiments.

Of course, there can still be false positive mistakes even with small
$\epsilon$ due to sampling randomness. We partially overcome this
issue by adding new candidate nodes into the ancestral set in the
\textit{decreasing order} of their $W$ statistics. This will help
ease the performance degradation for non-source nodes that are chosen
into the layer but still have low $W$ statistics.

Resultantly, the causal order will be the concatenation of the layers
collected after each step. Algorithm~\ref{alg:Causal-Ordering-Algorithm.}
summarizes the main steps of the Causal Ordering algorithm with all
the technical considerations discussed.

 Nevertheless, it is important to mention that our algorithm can still perform effectively even without utilizing the layer decomposition procedure. Furthermore, our layer decomposition only necessitates a single hyperparameter, the tolerance $\epsilon$. If this hyperparameter is set to zero, the feature is disabled, and this could be set as the default behavior for less experienced users of our algorithm.
 
\subsection{DAG Recovery From the Causal~Order\label{subsec:DAG-recovery}}

As stated in Lemma~\ref{lem:Faithfulness}, one can obtain the full
causal DAG from any valid causal order $\pi$. This is done by first
starting with an empty DAG, then examining every ordered pair of vertices
$\left(\pi_{j},\pi_{i}\right)$ to see if there is a directed edge
connecting them with the help of a series of CI tests.

Alternatively, instead of CI tests, one can employ feature selection
techniques with lower computational demands if runtime is preferred.
For example, GAM (Generalized Additive Model) feature selection is
widely used in several ordering-based methods that follow nonlinear
additive noise models, e.g., \cite{Buhlmann_etall_14Cam,Rolland_etal_22Score,Sanchez_etal_22Diffusion}.
More specifically, a GAM model is fitted to $\left(X_{\pi<i},X_{\pi_{i}}\right)$,
then any significant feature $X_{\pi_{j}}$ with $j<i$, evidenced
by a small $p$-value, will add an edge $\left(\pi_{j}\rightarrow\pi_{i}\right)$
to an initially empty DAG. Refer to \cite{Buhlmann_etall_14Cam} for
more details on this procedure.

Finally, we put together all the pieces of our \textbf{HOST} method
in Algorithm~\ref{alg:HOST}. The computational complexity of \textbf{HOST}
is given in Appendix~\ref{subsec:Complexity-analysis}, which is
polynomial in both sample size and dimensionality.

\section{Theoretical Properties}\label{sec:Theory}

\subsection{Identifiability of HCMs}
The identifiability of HCM is given by the following Theorem.
\begin{thrm}
    Under assumptions \ref{assu:regularity} and \ref{assu:stable-hcm}, HCMs are fully identifiable, i.e., there exists a unique graph $\mathcal{G}$ consistent with data $P_X$ generated according to Eqn.~\ref{eq:hcm}.
\end{thrm}
\begin{proofsketch}
    By Lemma~\ref{lem:conditional-normality} a valid source node can be identified at every step, thus by induction a valid causal ordering can be identified. Then by Lemma~\ref{lem:Faithfulness} every true edge can be identified from the causal ordering, rendering the whole DAG identifiable.
\end{proofsketch}

\subsection{Invariance to Scaling and Translation}

Another important property of our method is the invariance to scaling and invariance, which is given by the following Theorem. 

\begin{thrm}
    The \textbf{HOST} algorithm is invariant to scaling and translation.
\end{thrm}
\begin{proofsketch}
If we scale and translate $X_i$ by a scale $c > 0$ and location $d$, i.e., $X'_i:=cX_i+d$, the new causal model will have all the structural assignments and noise variables unchanged, except for $X_i$. Instead, we can replace it with $X'_i:=\mu'_i(X_{\textrm{pa}_i})+\sigma'_i(X_{\textrm{pa}_i})E_i$, where $\mu'_i:=\mu_i\times c+d$ and $\sigma'_i:=\sigma_i\times c$. In case $c < 0$, the negation can be absorbed to $E'_i:=-E_i$, which is still a standard Gaussian noise mutually independent of other noises.

Note that both $\mu'_i$ and $\sigma'_i$ maintain the properties of $\mu_i$ and $\sigma_i$ as described in Assumption~\ref{assu:regularity}, namely being deterministic and differentiable, $\sigma'_i>0$, and are non-constant with respect to any parent variable $j\in\textrm{pa}_i$. 

Additionally, since scaling nor translation does not alter normality, Assumption~\ref{assu:stable-hcm} also applies to the new HCM. Consequently, the newly obtained HCM is also identifiable using our method.
\end{proofsketch}

This property allows us to standardizing data before applying the \textbf{HOST} algorithm without changing the result. This is useful since neural networks, which are employed in our framework, are sensitive to data scale, meaning if the original data scale is too large or too small, the learning process may struggle to converge.

We note that this step is not possible with methods based on the ``equal variance'' assumption (e.g., EqVar \cite{Chen_etal_19Causal} or NPVar \cite{Gao_etal_20Polynomial}), since scaling breaks the equality between the variables' variances, making their algorithms misbehave.

\section{Numerical Evaluations\label{sec:Results}}

\subsection{Experiment Setup}

\paragraph{Baselines.}

We evaluate both the causal ordering and causal structure learning
performance of the proposed \textbf{HOST} method with competitive
baselines in ordering-based causal structure learning, including:
\begin{description}
\item [{VarSort~\cite{Reisach_etall_21Beware}}] simply sorts the variables
according to their marginal variances based on the observation that
an effect usually has higher variance than its causes in the common
evaluation practice of structure learning methods on simulated data.
This acts as a sanity check for the complexity of our problem.
\item [{EqVar~\cite{Chen_etal_19Causal}}] models the data with linear
relationships and homoscedastic additive noises of equal variances.
At each step, the variables with minimum conditional variances given
the current ancestral set are chosen as next source nodes.
\item [{NPVar~\cite{Gao_etal_20Polynomial}}] extends EqVar by considering
nonlinear causal mechanisms using nonparametric regression techniques.
It is claimed to be able to handle heteroscedasticity, but only with
constant expected noise variances, i.e., $\mathbb{E}\left[\mathbb{V}\left[X_{i}\mid X_{\pa_{i}}\right]\right]$.
\item [{DiffAN~\cite{Sanchez_etal_22Diffusion}}] considers nonlinear
additive noise models with homoscedastic Gaussian noises. Variables
with constant partial derivatives of the score function, learned with
diffusion models, are selected as sink nodes every step.
\end{description}

Additionally, for completeness, in Appendix~\ref{subsec:Additional-Experiments} we also compare \textbf{HOST} with other popular baselines that are not polynomial-time, including CAM \cite{Buhlmann_etall_14Cam}, GOLEM \cite{Ng_etal_2020Role}, and GraN-DAG \cite{Lachapelle_etal_2020Gradient}.

\begin{figure*}[t]
\begin{centering}
\begin{tabular}{c}
\includegraphics[width=1\textwidth]{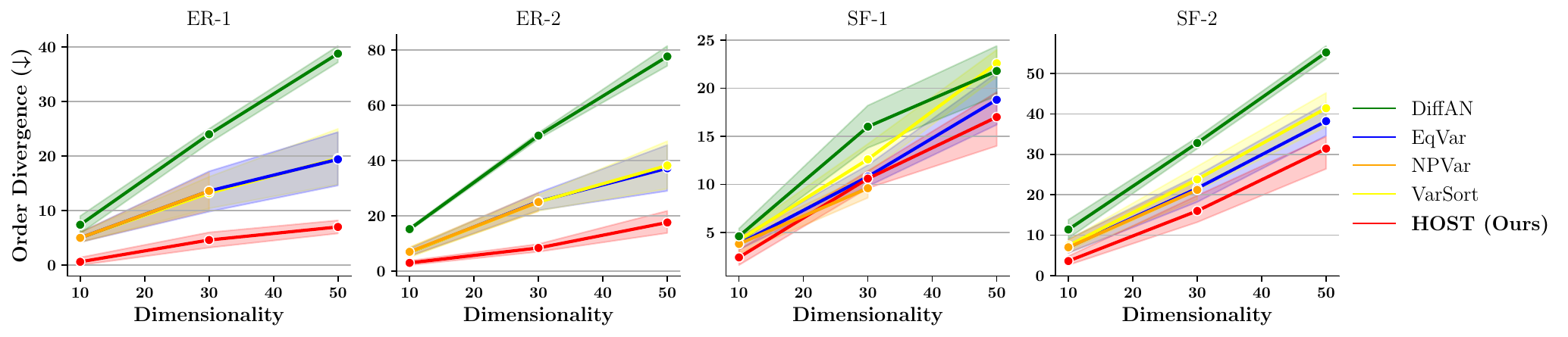}\tabularnewline
(a) Order Divergence (lower is better) with different Dimensionalities.\tabularnewline
\includegraphics[width=1\textwidth]{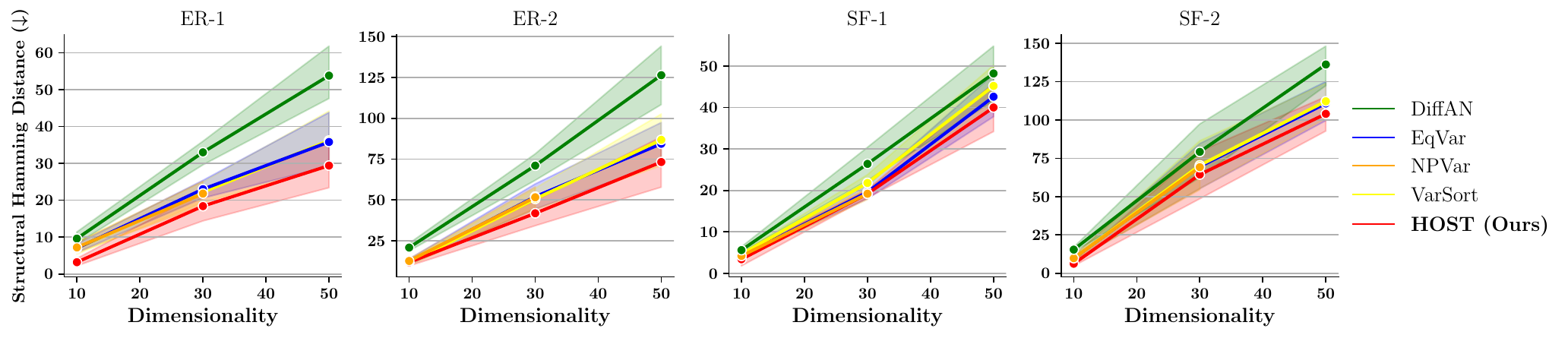}\tabularnewline
(b) Structural Hamming Distance (lower is better) with different Dimensionalities.\tabularnewline
\end{tabular}
\par\end{centering}
\centering{}\caption{Causal structure learning performance on synthetic data as function
of Dimensionality under Linear parameterization. We fix $n=500$ and
vary the dimensionality. The proposed \textbf{HOST} method is compared
against DiffAN \cite{Sanchez_etal_22Diffusion}, EqVar \cite{Chen_etal_19Causal},
NPVar \cite{Gao_etal_20Polynomial}, and VarSort \cite{Reisach_etall_21Beware}. Column: Graph
type. Shaded areas are 95\% confidence intervals over five independent
runs. LCIT~\cite{Duong_Nguyen_22Conditional} is used to recover
the DAGs from the causal orders. Missing data of NPVar is due to overly
excessive runtime.\label{fig:dimensionality-linear}}
\end{figure*}

\paragraph{Metrics.}

For assessing the causal ordering accuracy, we employ the Order Divergence
measure proposed by \cite{Rolland_etal_22Score}, which counts the
number of edges that are incorrectly ordered:
\[
\text{OrderDivergence}\left(\pi,\mathcal{E}\right)=\sum_{j<i}\mathds{1}\left[\left(\pi_{i}\rightarrow\pi_{j}\right)\in\mathcal{E}\right]
\]

Regarding causal structure learning, we adopt the common metric of
Structural Hamming Distance, which measures the number of edge additions,
removals, and reversals to transform the predicted DAG to the ground
truth DAG.

In addition, to tackle scenarios where there is an imbalance between missing and extra edges, we also employ the $F_1$ score and AUC (Area Under the Receiver Operating Characteristic Curve) measures, which consider both false negatives and false positives. The performance of our proposed method under these metrics is provided in Appendix~\ref{subsec:Additional-Experiments}, demonstrating the superiority of our proposed \textbf{HOST} method over the baselines in terms of both $F_1$ score and AUC. This finding is consistent with the results obtained using the SHD measure given in the next subsection.

\paragraph{DAG recovery methods.}

To remove the influence of the choice of the DAG recovery method from
the orderings onto the structure learning performance, we use the
same algorithm for all methods. Particularly, we consider two options,
first is the CI testing approach using the recent state-of-the-art
test LCIT \cite{Duong_Nguyen_22Conditional} since it is generic and
linearly scalable in both dimensionality and sample size, along with
the GAM feature selection approach which is usually employed in methods
based on the equal variances assumption \cite{Buhlmann_etall_14Cam,Rolland_etal_22Score,Sanchez_etal_22Diffusion}.
Both of them use $\alpha=0.001$.

\paragraph{Synthetic data.}

We generate data under two typical DAG settings. First is the Erd\H{o}s-Rényi
(ER) graph \cite{Erdos_Renyi_60Evolution}, where edges are independently
added with an equal probability, and we control the expected in-degree
to be one (ER-1 graphs) or two (ER-2 graphs). Second is the Scale
Free (SF) graph \cite{Barabasi_Albert_99Emergence}, with the SF-1
(SF-2) variant being the DAG initially started with one (two) nodes
and every subsequent node is added with one (two) random edges from
the previously added nodes. We also consider the scenarios with denser graphs (ER-4 and SF-4) in Appendix~\ref{subsec:Additional-Experiments}.

About the functional mechanisms, we generate $\mu_{i}\left(\cdot\right)$
and $\sigma_{i}\left(\cdot\right)$ from the natural parametrizations
and consider both linear and nonlinear scenarios. For the linear case,
$\eta_{1}$ and $\ln\left(-2\eta_{2}\right)$ are homogeneous linear
maps $a^{\top}x_{\pa_{i}}$ for each node. Meanwhile, in the nonlinear
case, these parameters associated with each parent are chosen from
the set $\left\{ a^{\top}x,x^{2},\sin\left(2\pi x\right),\ln\left(x-\min\left(x\right)+1\right),\frac{1}{1+e^{-x}}\right\} $,
and then summed afterward, e.g., $\eta_{1}\left(x_{\pa_{i}}\right)=\sum_{j\in\pa_{i}}\eta_{1}\left(x_{j}\right)$.

We emphasize here that Assumption~\ref{assu:stable-hcm} is not imposed in our data generating process, which means that it is possible for it to be violated in the generated data. However, even without enforcing this assumption, our method still shows robustness compared to the baselines in situations where the model is not well-specified, as evidenced by the empirical results reported in the next subsection.

\subsection{Results on Synthetic Data}

Here we present the empirical results for the linear parametrization
setting with LCIT being the DAG recovery method. The additional experimental results, including the consideration of other evaluation metrics, nonlinear settings, GAM feature selection, denser graphs, as well as the recorded runtimes can be found in Appendix~\ref{subsec:Additional-Experiments}.

\paragraph{Effect of sample size.}

We first study the influence of sample size to the learning performance
of the considering methods by fixing the number of nodes at $d=10$
and vary the sample size $n$ from 500 to 50,000 (Figure~\ref{fig:sample-size-linear}).
The results suggest that our method is far more effective than the
baseline counterparts across all scenarios, especially in the task
of causal ordering. Additionally, it can be observed that our method
converges to zero error on both metrics as the sample size increases,
empirically suggesting its consistency.

\paragraph{Effect of dimensionality.}

Next, we study the performance variation of all methods when the dimensionality
changes. To this end, we fix the sample size at $n=500$ and vary
the number of nodes from 10 to 50 (Figure~\ref{fig:dimensionality-linear}).
In this setting, while all methods show the same degradation in performance,
our proposed \textbf{HOST} method is still the leading performer over
all aspects.

\begin{table}
\begin{centering}
\caption{Causal structure learning performance on real data. We compare the
proposed \textbf{HOST} with DiffAN \cite{Sanchez_etal_22Diffusion},
EqVar \cite{Chen_etal_19Causal}, NPVAR \cite{Gao_etal_20Polynomial},
and VarSort \cite{Reisach_etall_21Beware} on the Sachs data set \cite{Sachs_etall_05Causal}.
The values are \textit{mean \textpm{} standard error} over the same
ten independent subsamples of size $700$.\label{tab:real-data}}
\par\end{centering}
\medskip{}

\centering{}\resizebox{\columnwidth}{!}{%
\begin{tabular}{c|ccc}
\toprule
 & Order & \multicolumn{2}{c}{Structural Hamming Distance ($\downarrow$)}\tabularnewline
 & Divergence ($\downarrow$) & (CI testing) & (GAM feature selection)\tabularnewline
\midrule
DiffAN & $9.4\pm0.70$ & $15.0\pm1.05$ & $15.6\pm0.84$\tabularnewline
EqVar & $7.2\pm0.42$ & \textbf{$\mathbf{13.5\pm0.71}$} & $\mathbf{12.0\pm0.47}$\tabularnewline
NPVar & $7.2\pm0.42$ & \textbf{$\mathbf{13.5\pm0.71}$} & $\mathbf{12.0\pm0.47}$\tabularnewline
VarSort & $7.2\pm0.42$ & \textbf{$13.7\pm0.95$} & $\mathbf{12.0\pm0.47}$\tabularnewline
\midrule
\textbf{HOST} (Ours) & $\mathbf{4.4\pm0.70}$ & $\mathbf{13.5}\pm0.97$ & $13.3\pm0.95$\tabularnewline
\bottomrule
\end{tabular}}
\end{table}

\subsection{Results on Real Data}

To demonstrate the effectiveness of our method in the real-life setting,
we conduct experiments on the well-known benchmark data set Sachs
\cite{Sachs_etall_05Causal}, where the ground truth causal network
is available. We employ the observational portion of the data set
with 853 observations, 11 vertices, and 17 edges in the causal graph.

Table~\ref{tab:real-data} displays the empirical results. Our method
achieves the best accuracy in recovering the causal order with an
error at only half of that for other methods, while it is competitive
with the state-of-the-arts in recovering the causal DAG.

\section{Conclusions\label{sec:Conclusions}}

This study presents the \textbf{HOST} algorithm for identifying the
causal structures under heteroscedastic causal models. By exploiting
the conditional normalities with the help of normality tests, we devise
a simple procedure for recovering causal orderings, which are used
to uniquely recover the causal structures. The empirical results on
a wide range of synthetic and real data show that \textbf{HOST} is
able to consistently outperform existing state-of-the-art ordering-based
methods in both causal ordering and structure learning.

\balance
\bibliography{host}

\onecolumn

\section*{Appendix for ``Heteroscedastic Causal Structure Learning''}

\renewcommand\thesubsection{\Alph{subsection}}

\subsection{The Shapiro-Wilk Test for Normality\label{subsec:Shapiro-Wilk-test}}

In the Shapiro-Wilk test, the test's statistics is the squared correlation
between the order statistics of the observations and the expected
order statistics of the samples following the Gaussian distribution.
More specifically, consider $n$ i.i.d. samples $\mathbf{u}_{i}=\left\{ u_{i}^{\left(k\right)}\right\} _{k=1}^{n}$
of $U_{i}$. Without loss of generality, assume they are sorted, i.e.,
$u_{i}^{\left(k\right)}\leq u_{i}^{\left(k+1\right)}\;\forall k=1..n-1$.
Similarly, let $\mathbf{Z}=\left\{ Z_{k}\right\} _{k=1}^{n}$ be i.i.d.
Gaussian random variables and also take their order statistics, i.e.,
$Z_{k}\leq Z_{k+1}\;\forall k=1..n-1$. Subsequently, define $\mathbf{m}=\left\{ m_{k}=\mathbb{E}\left[Z_{k}\right]\right\} _{k=1}^{n}$
as the expectation of the order statistics.

Now, under the null hypothesis, $\mathbf{u}_{i}$ and $\mathbf{m}$
should be strongly correlated, meaning a correlation coefficient close
to one would suggest $U_{i}$ is normally distributed, and a number
closer to zero would indicate non-normality. The straightforward squared
correlation between these two sample sets is referred to as the Shapiro-Francia
statistics \cite{Shapiro_Francia_72Approximate}.

However, the Shapiro-Wilk statistics goes one step further as it also
takes into account the covariance matrix of $Z$. More specifically,
consider the covariance matrix $\Sigma$ of $Z$ where $\Sigma_{ij}=\mathbb{E}\left[\left(Z_{i}-m_{i}\right)\left(Z_{j}-m_{j}\right)\right]$
and let $\mathbf{a}=\frac{m^{\top}\Sigma^{-1}}{\left\Vert m^{\top}\Sigma^{-1}\right\Vert _{2}}$,
which is a unit vector. Then, the Shapiro-Wilk test statistics is
given by the squared correlation between $\mathbf{a}$ and $\mathbf{u}_{i}$:
\begin{equation}
W=\frac{\left(\sum_{k=1}^{n}a_{k}u_{i}^{\left(k\right)}\right)^{2}}{\sum_{k=1}^{n}\left(u_{i}^{\left(k\right)}-\bar{u_{i}}\right)^{2}}\label{eq:W-statistics}
\end{equation}

\subsection{Complexity Analysis\label{subsec:Complexity-analysis}}

\paragraph{Algorithm~\ref{alg:extract-residuals}~(Latents~Extraction).}

With gradient-based methods such as vanilla Stochastic Gradient Ascent
or Adam, the optimization process can be done in $\mathcal{O}\left(nd\right)$,
with $n$ and $d$ being sample size and dimensionality, respectively.
With more advanced optimizer for concave loss function, such as Newton
methods or Feasible Generalized Least-squares \cite{Immer_etal_22Identifiability},
the time complexity can be furthermore reduced.

\paragraph{Algorithm~\ref{alg:Causal-Ordering-Algorithm.}~(Causal~Ordering).}

The main loop contains $\mathcal{O}\left(d\right)$ iterations. In
each iteration, Algorithm~\ref{alg:extract-residuals} and $W$ statistics
evaluation are employed $\mathcal{O}\left(d\right)$ times. To evaluate
the $W$ statistics, at least $\mathcal{O}\left(n^{3}\right)$ of
runtime is needed due to the inverse of $\Sigma\in\mathbb{R}^{n\times n}$,
however, practical approximations for large $n$ exist and can scale
very well in $n$, such as the SciPy package \cite{Virtanen_etal_20Scipy}
which is used in our implementation. Additionally, sorting $L$ can
be done in $\mathcal{O}\left(d\ln d\right)$ time. Assuming $n>d^{2}$
for simplicity, then Algorithm~\ref{alg:Causal-Ordering-Algorithm.}
operates in $\mathcal{O}\left(n^{3}d^{2}\right)$.

\paragraph{Algorithm~\ref{alg:HOST}~(HOST).}

There are $\mathcal{O}\left(d^{2}\right)$ CI tests to be performed.
CI tests that scale linearly with sample size and dimensionality of
the conditioning set exists, e.g., \cite{Duong_Nguyen_22Conditional}.
Thus, the DAG recovery step can be completed in $\mathcal{O}\left(nd^{3}\right)$
runtime, so the total complexity of the \textbf{HOST} algorithm is
$\mathcal{O}\left(nd^{2}\left(n^{2}+d\right)\right)$.

\subsection{Additional Experiments\label{subsec:Additional-Experiments}}

We provide empirical results of additional settings, including:
\begin{itemize}
\item Figure~\ref{fig:samplesize-F1-AUC} and Figure~\ref{fig:dimensionality-F1-AUC}: We study the causal structure learning performance of \textbf{HOST} and competing methods under $F_1$ score and AUC metrics. Our method is able to outperform all baseline methods under all metrics.

\item Figure~\ref{fig:denser-graph}: We study the causal structure learning performance of \textbf{HOST} and baselines under denser graphs settings (ER-4 and SF-4). The results, combined with those in Figure~\ref{fig:sample-size-linear}, suggest that as the graph density increases, all methods experience a proportional degradation in performance. However, our method still outperforms the others by a significant margin in all cases, as evidenced by the graph.

\item Figure~\ref{fig:runtime}: We compare the runtime of all methods
in the linear parametrization setting under the metric of causal ordering
time, since the DAG recovery step is performed similarly with a fixed
algorithm. Apart from the simple methods VarSort and EqVar which theoretically
have lower computational complexities, our method is faster than both
NPVar and DiffAN, except in the small sample size settings, which
indicates the scalability of our method.

\item Figure~\ref{fig:samplesize-linear-gam} and Figure~\ref{fig:dimensionality-linear-gam}:
Causal structure learning performance in the linear parametrization
setting with GAM feature selection as the DAG recovery method.

\item Figure~\ref{fig:nonlinear}: Causal structure learning performance
in the nonlinear parametrization setting with LCIT as the DAG recovery
method. In this setting we use MLPs with one hidden layer of four
units for $\eta_{1}$ and $\ln\left(-2\eta_{2}\right)$.

\item Figure~\ref{fig:non-polynomial-time-methods}: We compare \textbf{HOST} with popular baselines that are not polynomial time, including CAM \cite{Buhlmann_etall_14Cam}, GOLEM \cite{Ng_etal_2020Role}, and GraN-DAG \cite{Lachapelle_etal_2020Gradient}.

\end{itemize}

\begin{figure*}[t]
\begin{centering}
\begin{tabular}{c}
\includegraphics[width=1\textwidth]{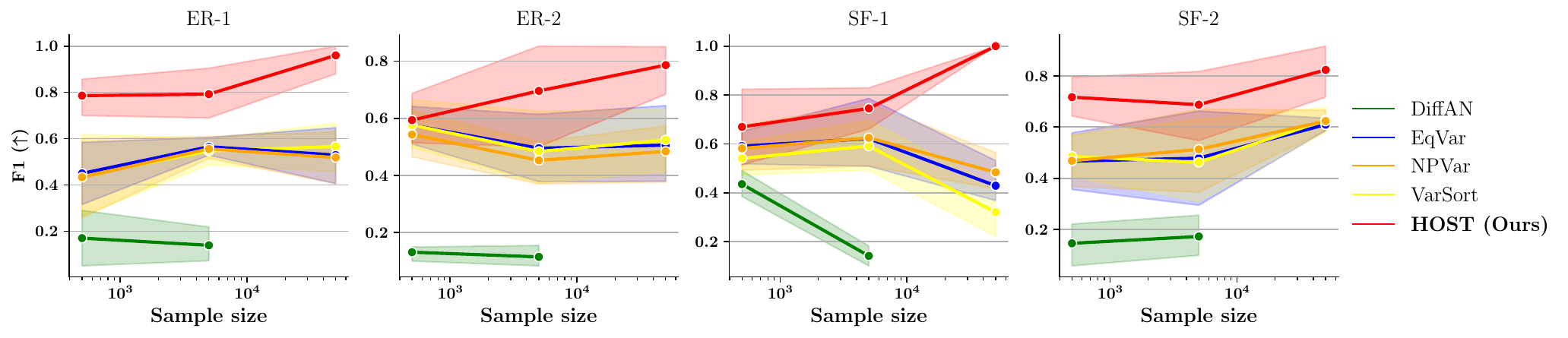}\tabularnewline
(a) $F_1$ score (higher is better) with different Sample sizes.\tabularnewline
\includegraphics[width=1\textwidth]{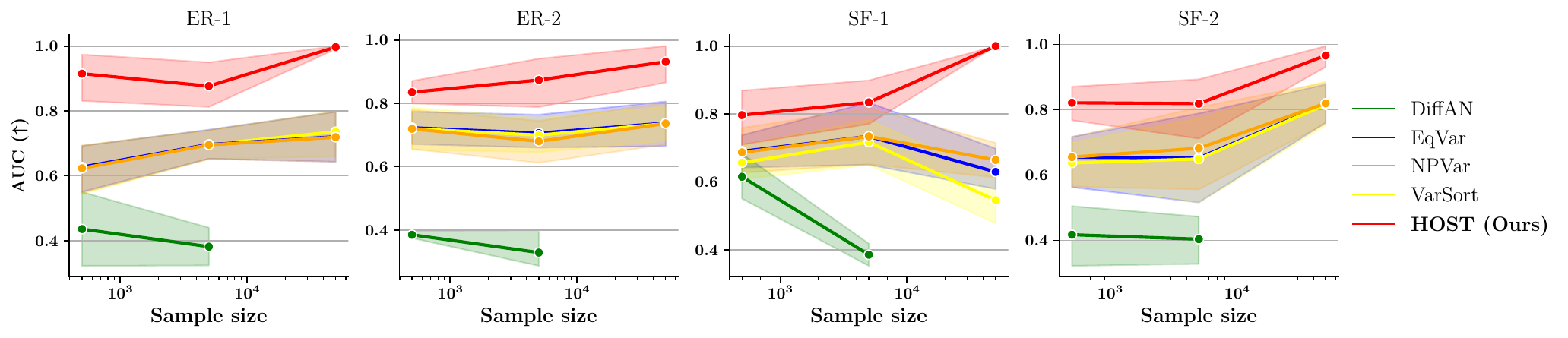}\tabularnewline
(b) AUC (higher is better) with different Sample sizes.\tabularnewline
\end{tabular}
\par\end{centering}
\centering{}\caption{Causal structure learning performance (in $F_1$ score and AUC) on synthetic data as function
of Sample size under Linear parameterization. We fix $d=10$ and
vary the sample size. The proposed \textbf{HOST} method is compared
against DiffAN \cite{Sanchez_etal_22Diffusion}, EqVar \cite{Chen_etal_19Causal},
NPVar \cite{Gao_etal_20Polynomial}, and VarSort \cite{Reisach_etall_21Beware}. Column: Graph
type. Shaded areas are 95\% confidence intervals over five independent
runs. LCIT~\cite{Duong_Nguyen_22Conditional} is used to recover
the DAGs from the causal orders. Missing data of NPVar is due to overly
excessive runtime.\label{fig:samplesize-F1-AUC}}
\end{figure*}

\begin{figure*}[t]
\begin{centering}
\begin{tabular}{c}
\includegraphics[width=1\textwidth]{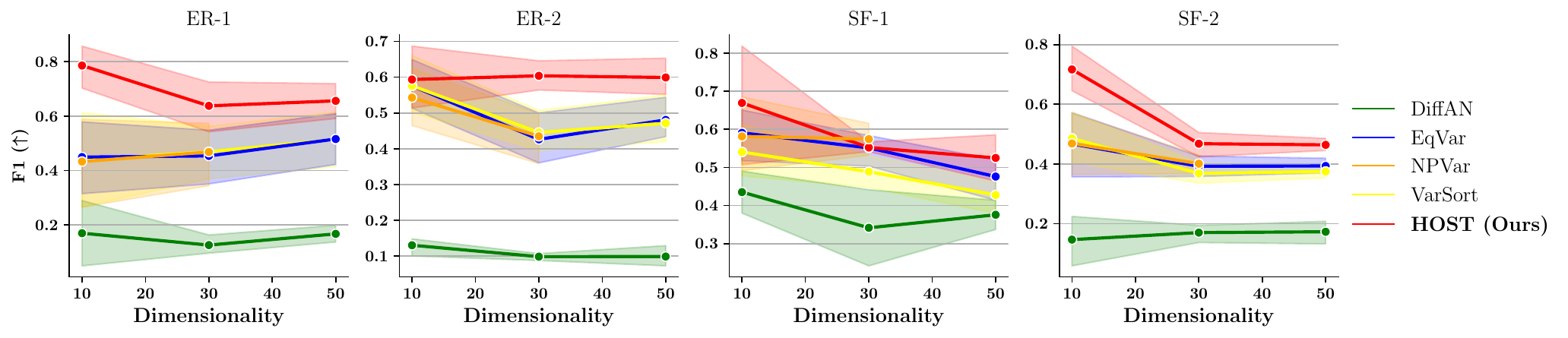}\tabularnewline
(a) $F_1$ score (higher is better) with different Dimensionalities.\tabularnewline
\includegraphics[width=1\textwidth]{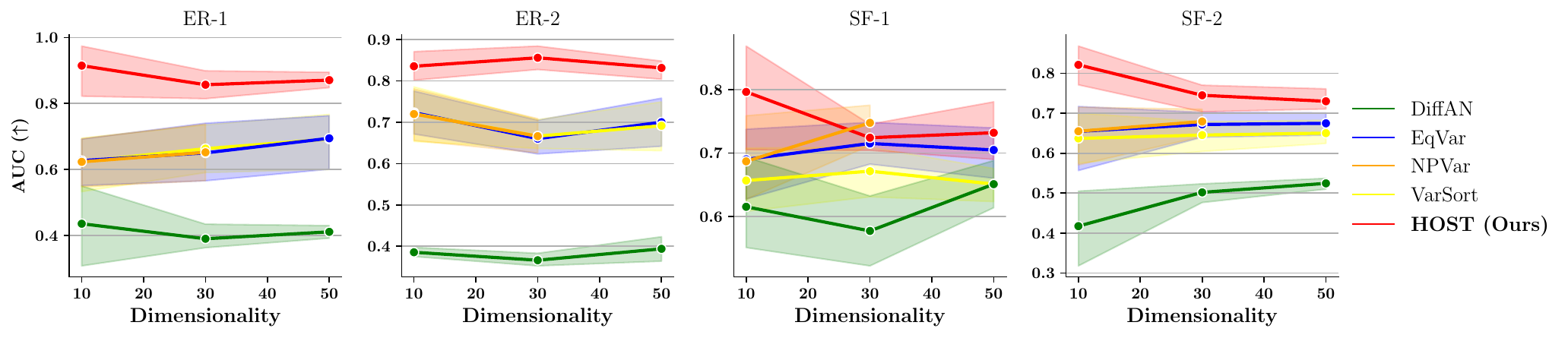}\tabularnewline
(b) AUC (higher is better) with different Dimensionalities.\tabularnewline
\end{tabular}
\par\end{centering}
\centering{}\caption{Causal structure learning performance (in $F_1$ score and AUC) on synthetic data as function
of Dimensionality under Linear parameterization. We fix $n=500$ and
vary the dimensionality. The proposed \textbf{HOST} method is compared
against DiffAN \cite{Sanchez_etal_22Diffusion}, EqVar \cite{Chen_etal_19Causal},
NPVar \cite{Gao_etal_20Polynomial}, and VarSort \cite{Reisach_etall_21Beware}
under the Linear parametrization setting. Column: Graph
type. Shaded areas are 95\% confidence intervals over five independent
runs. Missing data of NPVar is due to overly
excessive runtime.\label{fig:dimensionality-F1-AUC}}
\end{figure*}

\begin{figure*}[t]
\centering{}%
\begin{tabular}{cc}
\includegraphics[width=0.5\columnwidth]{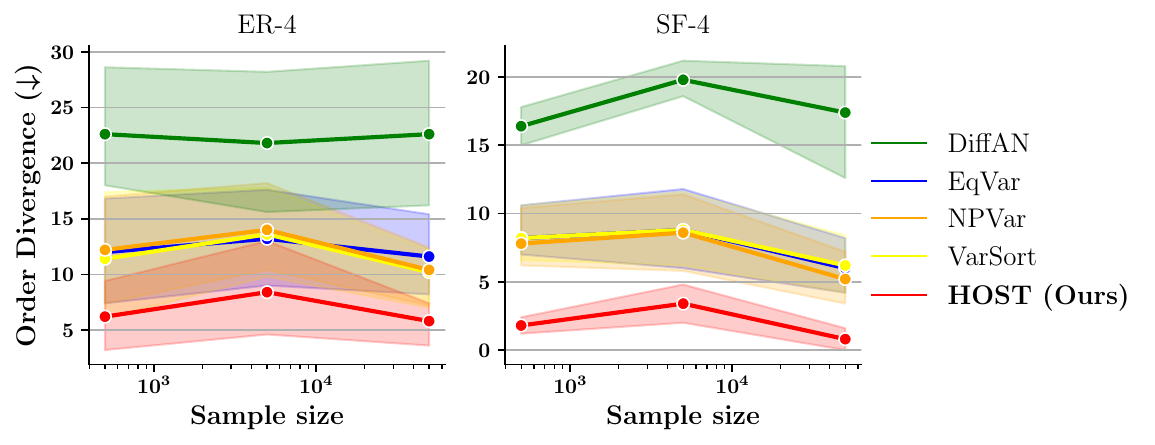} & \includegraphics[width=0.5\columnwidth]{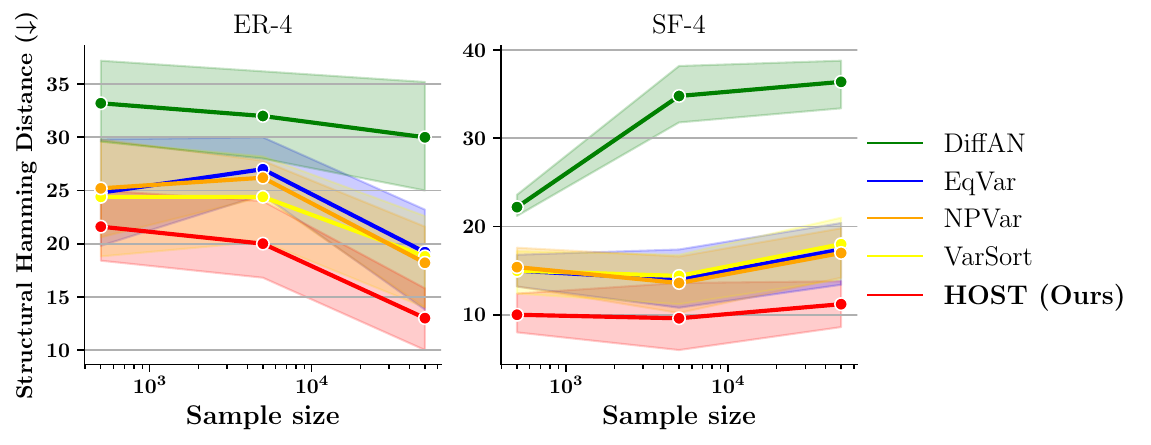}\tabularnewline
(a) Causal Ordering & (b) Causal Structure Learning\tabularnewline
\end{tabular}\caption{Causal structure learning performance on synthetic data under denser graph settings. We fix $n=5000$ and $d=10$.
The proposed \textbf{HOST} method is compared against DiffAN \cite{Sanchez_etal_22Diffusion},
EqVar \cite{Chen_etal_19Causal}, NPVar \cite{Gao_etal_20Polynomial},
and VarSort \cite{Reisach_etall_21Beware} under the Linear parametrization setting.
Column: Graph type. Error bars are 95\% confidence intervals over
five independent runs.\label{fig:denser-graph}}
\end{figure*}

\begin{figure}
\begin{centering}
\includegraphics[width=0.7\columnwidth]{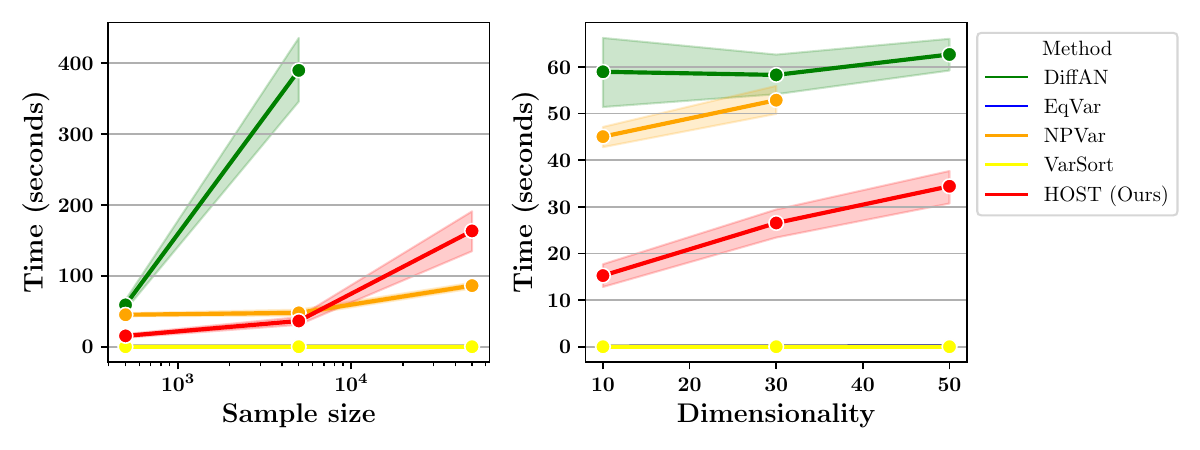}
\par\end{centering}
\centering{}\caption{Causal ordering runtime. Left: we fix $d=10$ and vary the sample
size; right: we fix $n=500$ and vary dimensionality. The proposed
\textbf{HOST} method is compared against DiffAN \cite{Sanchez_etal_22Diffusion},
EqVar \cite{Chen_etal_19Causal}, NPVar \cite{Gao_etal_20Polynomial},
and VarSort \cite{Reisach_etall_21Beware} under the Linear parametrization setting.
Shaded areas are 95\% confidence intervals over five independent runs.
Missing numbers of DiffAN and NPVar are due to significantly higher
runtimes compared with other competitors.\label{fig:runtime}}
\end{figure}

\begin{figure*}[t]
\begin{centering}
\includegraphics[width=1\textwidth]{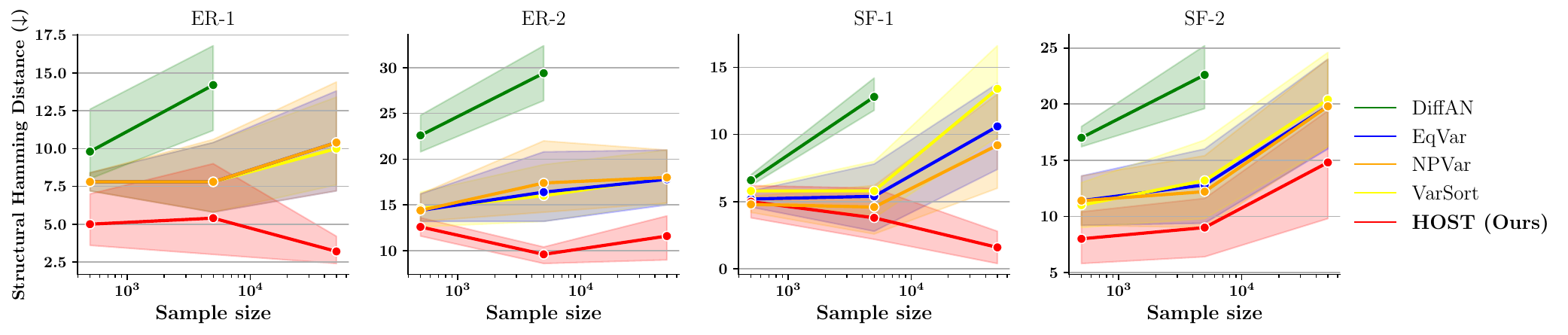}\caption{Causal structure learning performance on synthetic data as function
of Sample size under Linear parameterization. GAM feature selection
is used to recover the DAGs from the causal orders. We fix $d=10$
and vary the sample size. The proposed \textbf{HOST} method is compared
against DiffAN \cite{Sanchez_etal_22Diffusion}, EqVar \cite{Chen_etal_19Causal},
NPVar \cite{Gao_etal_20Polynomial}, and VarSort \cite{Reisach_etall_21Beware}. Column: Graph
type. Shaded areas are 95\% confidence intervals over five independent
runs. LCIT~\cite{Duong_Nguyen_22Conditional} is used to recover
the DAGs from the causal orders.\label{fig:samplesize-linear-gam}}
\par\end{centering}
\end{figure*}

\begin{figure*}[t]
\centering{}\includegraphics[width=1\textwidth]{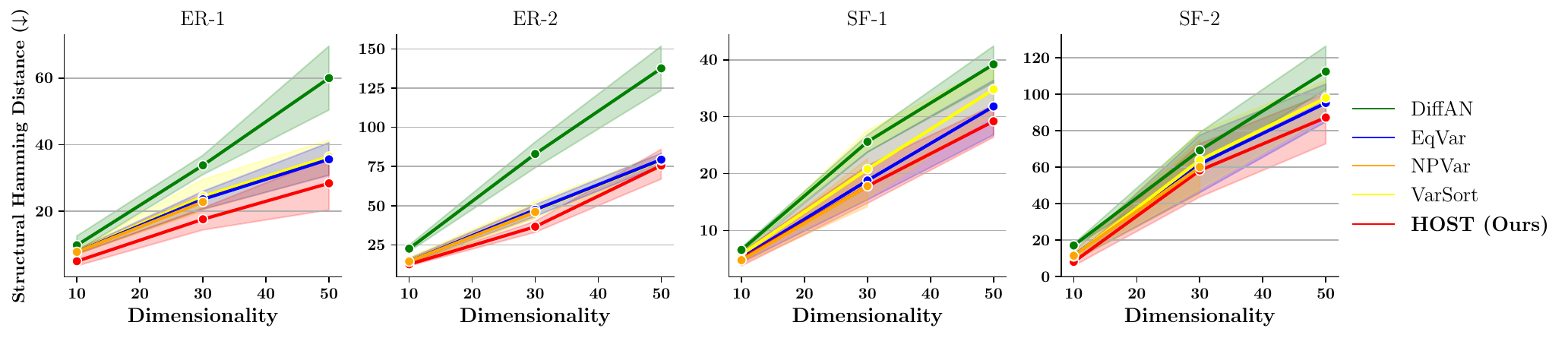}\caption{Causal structure learning performance on synthetic data as function
of Dimensionality under Linear parameterization. GAM feature selection
is used to recover the DAGs from the causal orders. We fix $n=500$
and vary the dimensionality. The proposed \textbf{HOST} method is
compared against DiffAN \cite{Sanchez_etal_22Diffusion}, EqVar \cite{Chen_etal_19Causal},
NPVar \cite{Gao_etal_20Polynomial}, and VarSort \cite{Reisach_etall_21Beware}. Column: Graph
type. Shaded areas are 95\% confidence intervals over five independent
runs. \label{fig:dimensionality-linear-gam}}
\end{figure*}

\begin{figure*}[t]
\centering{}%
\begin{tabular}{cc}
\includegraphics[width=0.5\columnwidth]{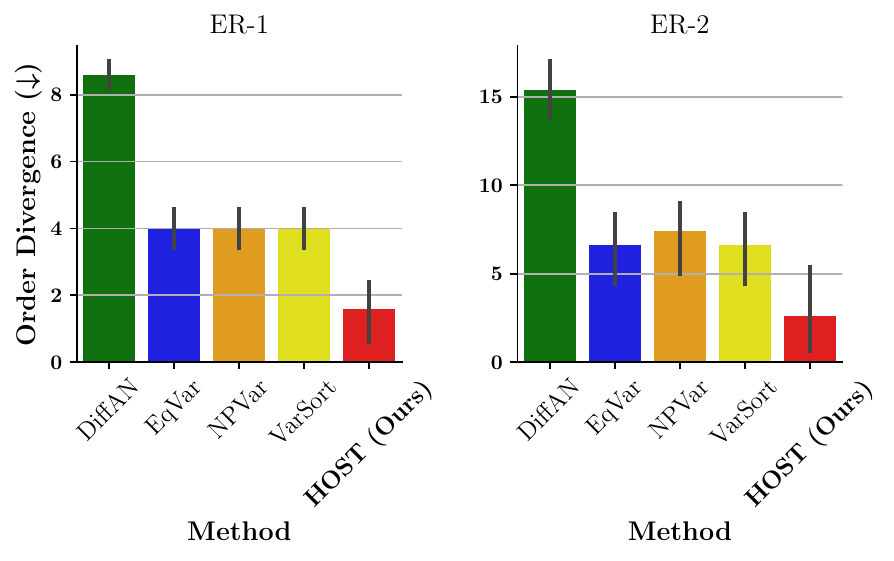} & \includegraphics[width=0.5\columnwidth]{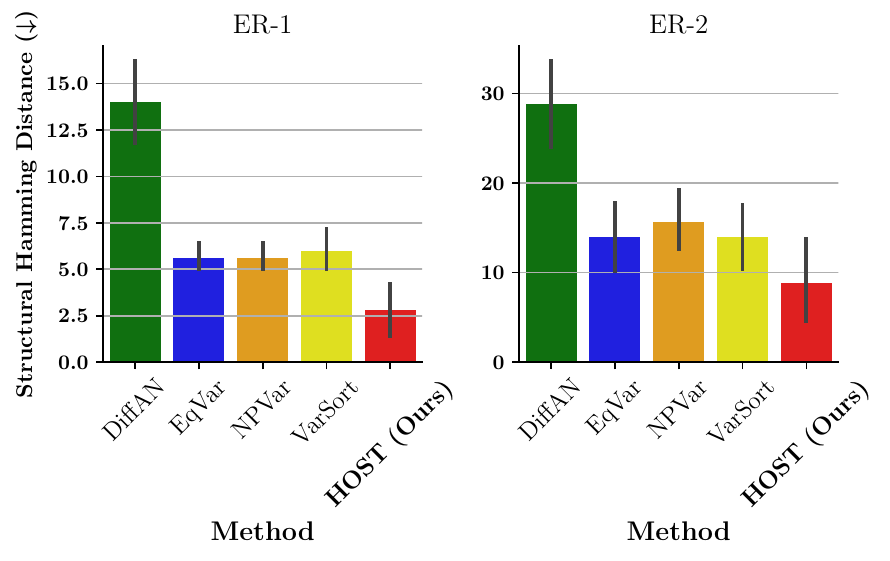}\tabularnewline
(a) Causal Ordering & (b) Causal Structure Learning\tabularnewline
\end{tabular}\caption{Causal structure learning performance on synthetic data under Nonlinear
parameterization. LCIT~\cite{Duong_Nguyen_22Conditional} is used
to recover the DAGs from the causal orders. We fix $N=5000$ and $d=10$.
The proposed \textbf{HOST} method is compared against DiffAN \cite{Sanchez_etal_22Diffusion},
EqVar \cite{Chen_etal_19Causal}, NPVar \cite{Gao_etal_20Polynomial},
and VarSort \cite{Reisach_etall_21Beware}.
Column: Graph type. Error bars are 95\% confidence intervals over
five independent runs.\label{fig:nonlinear}}
\end{figure*}

\begin{figure}[t]
\begin{centering}
\includegraphics[width=0.7\textwidth]{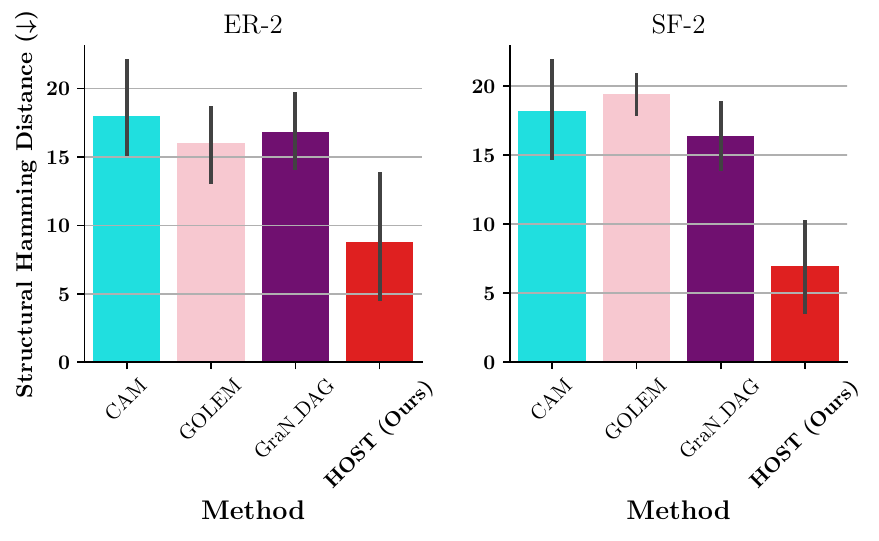}
\end{centering}
\centering{}\caption{Causal structure learning performance on synthetic data as function
of Dimensionality under Linear parameterization in comparision with non-polynomial time methods. We fix $n=500$ and $d=10$. The proposed \textbf{HOST} method is compared
against CAM \cite{Buhlmann_etall_14Cam}, GOLEM \cite{Ng_etal_2020Role}, and GraN-DAG \cite{Lachapelle_etal_2020Gradient}. Column: Graph
type. Error bars are 95\% confidence intervals over five independent
runs.\label{fig:non-polynomial-time-methods}}
\end{figure}

\end{document}